%% file: camera-ready.tex
\documentclass[letterpaper]{article} 
\usepackage{aaai24}  
\usepackage{times}  
\usepackage{helvet}  
\usepackage{courier}  
\usepackage[hyphens]{url}  
\usepackage{graphicx} 
\usepackage{natbib}  
\usepackage{caption} 
\usepackage{algorithm}
\usepackage[noend]{algorithmic}

\usepackage{newfloat}
\usepackage{listings}
\DeclareCaptionStyle{ruled}{labelfont=normalfont,labelsep=colon,strut=off} 
\floatstyle{ruled}
\newfloat{listing}{tb}{lst}{}
\floatname{listing}{Listing}
\title{Expected Runtime Comparisons Between \\Breadth-First Search and Constant-Depth Restarting Random Walks}

\author{
    Daniel Platnick and Richard Valenzano
}
\affiliations{
    Toronto Metropolitan University, Toronto, Canada\\
    Vector Institute for Artificial Intelligence, Toronto, Canada\\
    
    \{daniel.platnick, rick.valenzano\}@torontomu.ca
}

\usepackage{caption}
\usepackage{subcaption}
\usepackage{amsmath}
\usepackage{amsfonts}
\usepackage{amsthm}
\usepackage{amssymb}
\newtheorem{theorem}{Theorem}[section]

\newtheorem{corollary}{Corollary}[section]

\input{macros}

\begin{document}

\pagestyle{plain} 
\maketitle

\maketitle

\begin{abstract}
When greedy search algorithms encounter a local minima or plateau, the search typically devolves into a breadth-first search (BrFS), or a local search technique is used in an attempt to find a way out.
In this work, we formally analyze the performance of BrFS and constant-depth restarting random walks (RRW) --- two methods often used for finding exits to a plateau/local minima --- to better understand when each is best suited.
In particular, we formally derive the expected runtime for BrFS in the case of a uniformly distributed set of goals at a given goal depth.
We then prove RRW will be faster than BrFS on trees if there are enough goals at that goal depth. 
We refer to this threshold as the crossover point.
Our bound shows that the crossover point grows linearly with the branching factor of the tree, the goal depth, and the error in the random walk depth, while the size of the tree grows exponentially in branching factor and goal depth.
Finally, we discuss the practical implications and applicability of this bound.
\end{abstract}

\section{Introduction}

Greedy algorithms like \emph{Greedy Best First Search (GBFS)} \cite{Doran:gbfs} and \emph{Enhanced Hill-Climbing (EHC)} \cite{h:01} have been shown to be very effective at solving graph search and planning problems.
However, when using a flawed heuristic function, the search performed by these algorithms often devolves into \emph{breadth-first search (BrFS)}.
This is explicitly by design in the case of EHC, which performs a sequence of BrFS instances, each in search of a new lower heuristic value.
In GBFS, it often happens as a natural consequence of \emph{plateaus}, where the heuristic function does not provide enough information to distinguish between states, and so the search is blind.

The \emph{restarting random walk (RRW)} approach \cite{nakhost:arvand,xie:local_rws,xie:improved_local} is an alternative to BrFS for escaping local minima.
This method generates a sequence of random walks until one is found that exits the local minima.
It is low memory, but does not have the strong termination guarantees that the systematic search of BrFS leads to. 

The main objective of this work is to further our understanding of BrFS and RRW to better understand when each is best for dealing with local minima and plateaus, when these are encountered as part of other search algorithms. 
As an example of why this is needed, we have performed preliminary experiments testing the performance of an EHC algorithm which replaces BrFS with RRWs. These experiments were performed using Pyperplan \cite{alkhazraji:pyperplan}.
In particular, we ran each method 10 times per problem, each time with a 10 minute time limit on an M3 Macbook with 16 GB of RAM\footnote{While Pyperplan is not a state-of-the-art planner, we consider it sufficient for demonstrating that the competing methods of using BrFS or RRW for escaping local minima each are best suited for different problems.}.
The results show that different methods for escaping local minima seem best suited for different domains.
Using RRWs in EHC instead of BrFS increases coverage in some domains, for example, from 60\% to 95\% in TPP, and 81\% to 91\% in Transport. In other domains, we see that using RRW decreases performance, such as the drop in coverage from 93\% to 76\% in Scanalyzer and 100\% to 84\% in BlocksWorld. 

In this paper, we focus on formally analyzing BrFS and RRW to better understand where each approach is best suited.
Our main result shows that the expected runtime of RRW on directed trees is less than BrFS if the number of goals (\textit{ie.} exit states of the local minima) is high enough.
Interestingly, the number of goals at this \emph{crossover point} is linear in the goal depth and branching factor, meaning it grows much slower than the overall size of the tree itself.
A consequence of this is that the density of the number of goals (\textit{ie.} the percentage of goals to total states) needed for RRW to outperform BrFS actually decreases as the depth of the goals increases.
We then conclude by discussing the practical implications of this result as it relates to when each algorithm is most appropriate.

\section{Preliminaries}

In this section, we introduce the terminology used and briefly summarize the algorithms covered in this work. 

A graph search task is defined by the tuple $\mathcal{S} = \tuple{G, \initvertex, \goalset}$, where  $G = \tuple{V, E}$ is a graph with vertices $V$ and edges $E$, $\initvertex \in V$ is the \emph{start} or \emph{initial} vertex and $\goalset \subseteq V$ is the \emph{goal} set.
The objective is to find a \emph{path} of vertices $\tuple{v_0, ..., v_k}$ where $v_0 = \initvertex$, $v_k \in \goalset$, and for every $0\leq i < k$, $\tuple{v_i, v_{i+1}} \in E$.
Given our specific focus algorithms, we ignore edge costs in this work.

For any $v \in V$, we refer to the number of edges in the shortest path to $v$ as the \emph{level} of $v$.
For example, $\initvertex$ has a level of $0$, any vertex that can be reached from $\initvertex$ by one edge has a level of $1$, etc. 
The \emph{goal level} $\goallevel$ is defined as the minimum level of any vertex in $\goalset$.

The search algorithms we focus on perform two main operations on vertices.
The first is a \emph{goal test} (\textit{ie.} if $v$ is in $\goalset$).
The second is a \emph{successor generation step}.
This means finding and generating the successor set of $v$, defined as $\lbrace v' | \tuple{v, v'} \in E \rbrace$.
When both of these operations are performed on a vertex, we say it has been \emph{expanded}.
As will be described below, for both our focus algorithms, the number of goal tests will be exactly $1$ larger than the number of successor generation steps.
Thus, we can measure runtime equally well using the count of the number of times either operation is performed.
We therefore focus on the number of goal tests in the remainder of this work.

Note that while our formal analysis is given in terms of finding a goal in a given graph search task, the analysis equally applies to the case of escaping from a single local minima in a given state-space.
In such a situation, the goal test involves simply checking if the heuristic value of a given state (\textit{ie.} vertex) is less than the heuristic value of the initial state.
Thus, the statements below also provide expected runtime results for the time needed to escape from local minima of different sizes.

\subsection{Breadth-First Search (BrFS)}

For BrFS, we assume the reader's familiarity with the use of open and closed lists in best-first search algorithms like A$^*$.
At every iteration, BrFS will select one of the vertices $v$ in the open list with the lowest level and perform a goal test.
If the test succeeds, the search immediately terminates.
If the test does not succeed, then the successors of $v$ are generated, those that have been generated for the first time are added to the open list, and $v$ is moved to the closed list.
This process then repeats until a goal test succeeds or the open list is emptied.
Thus, when we refer to BrFS, we are considering an algorithm equivalent to uniform-cost search \cite{felner:uniform_cost_search} --- or A$^*$ with a heuristic  that always returns $0$ --- on a unit-cost graph.

By definition, BrFS expands all vertices in a level before proceeding to the next level (assuming no goal is found), and BrFS does not perform a goal test on any vertex with a level larger than $\goallevel$.
Since every failed goal test is immediately followed by a successor generation step and BrFS terminates after the test succeeds once, the number of goal tests is exactly one more than the number of successor generation steps.
Moreover, the number of goal tests will depend on how many vertices in the goal level are expanded before a goal vertex, which will depend entirely on how the goal vertices are distributed at that level and tiebreaking.
Thus, the number of goal tests made by BrFS on a search task $\mathcal{S}$ is a random variable, which we denote by $\bfsrv(\mathcal{S})$.

\begin{algorithm}[t]
\caption{The Constant-Depth RRW Algorithm}\label{alg:rrw}
\hspace*{\algorithmicindent} \textbf{Input: } task $\mathcal{S} = \tuple{G = \tuple{V, E}, \initvertex, \goalset}$, max depth $t$

\begin{algorithmic}[1]
    \IF{$\initvertex \in \goalset$} \label{rrw_alg:init_goal_test}
        \STATE \textbf{return} path $\tuple{\initvertex}$
    \ENDIF
    \WHILE{\textbf{true}}
        \STATE  $P \gets \tuple{\initvertex}$, $v \gets \initvertex$, $d \gets 0$ \label{rrw_alg:rw_start}
        \WHILE{$d < t$} \label{rrw_alg:restart}
            \STATE $v \gets$ vertex sampled from $\lbrace v' | \tuple{v, v'} \in E \rbrace$ \label{rrw_alg:successor_gen}
            \STATE Append $v$ to $P$, $d += 1$ 
            \IF{$v \in \goalset$} \label{rrw_alg:goal_test}
                \STATE \textbf{return} $P$ \label{rrw_alg:rw_end}
            \ENDIF
        \ENDWHILE
    \ENDWHILE
\end{algorithmic}
\end{algorithm}

\subsection{Constant-Depth Restarting Random Walks (RRW)}

The constant-depth RRW algorithm is shown in Algorithm \ref{alg:rrw}.
It generates a series of random paths called \emph{random walks}, each starting from $\initvertex$ (lines \ref{rrw_alg:rw_start} to \ref{rrw_alg:rw_end}).
At every step of the walk, the algorithm terminates if a goal is found (line \ref{rrw_alg:goal_test}).
Otherwise, the successor of the last vertex is generated, and it is added to the path.
The algorithm restarts to the initial vertex when the maximum depth is reached (lines \ref{rrw_alg:restart} and \ref{rrw_alg:rw_start}).

Notice that the goal test is performed on $\initvertex$ exactly once (line \ref{rrw_alg:init_goal_test}).
Doing otherwise is clearly unnecessary.
This approach also ensures that the algorithm performs exactly one more goal test than successor generation steps, since every successor generation step (line \ref{rrw_alg:successor_gen}) is immediately followed by a goal test (line \ref{rrw_alg:goal_test}).

Clearly, if the maximum depth $t$ is less than the goal level, then constant-depth RRW will not terminate. 
Thus, we assume a reasonable ``guess" of a lower bound on the correct depth can be made.
That is, we assume that $t = e \goallevel$ where $e \geq 1$ and $e \goallevel$ is an non-negative integer.
We refer to $e$ as the \emph{depth error}, and denote the random variable for the number of goal tests by a constant-depth RRW with depth error $e$ on search task $\mathcal{S}$ as $\rrwrv_e(\mathcal{S})$. 
Note that for simplicity, we assume that every vertex with a level less than $e \goallevel$ has at least one successor (\textit{ie.} a random walk never terminates early).

\section{General Expected Runtime Analysis}

In this section, we characterize the expected runtime of BrFS and RRW in general settings. The first result is for BrFS when the goal vertices are uniformly distributed at the goal level.
This theorem will equivalently hold if random tiebreaking is used to order the vertices for expansion.

\begin{theorem} \label{thm:brfs_exp}
If $\mathcal{S}$ has $N_{\mathcal{O}}$ unique vertices with a level less than $\goallevel$, and $\numgoals \geq 1$ goal vertices uniformly distributed amongst the $\goallevelsize \geq \numgoals$ unique vertices at the goal level, then
\begin{align*}
    \expect[\bfsrv(\mathcal{S})] = N_{\mathcal{O}} + (\goallevelsize + 1)/(\numgoals + 1)
\end{align*}
\end{theorem}
\begin{proof}
Let $X$ be the number of goal tests that BrFS performs on vertices at level $\goallevel$.
Since BrFS does not test any vertices with a level deeper than $\goallevel$, $\expect[\bfsrv(\mathcal{S})] = \expect[N_{\mathcal{O}} + X] = N_{\mathcal{O}} + \expect[X]$ since $N_{\mathcal{O}}$ is a constant.

Note, because the goal vertices are uniformly distributed, $\expect[X]$ is equivalent to the expected number of selections needed when randomly picking vertices from the goal level without replacement, until one of the $\numgoals$ vertices is picked. 
Let $v_i$ be any one of the $(\goallevelsize - \numgoals)$ non-goal vertices at level $\goallevel$, and let $Z_i$ be an indicator random variable for the event that $v_i$ is picked before any one of the $\numgoals$ goal vertices.
Therefore $\expect[X] = \expect[Z_1 + ... + Z_{\goallevelsize - \numgoals}] + 1$, since $\expect[X]$ is the number of non-goal vertices tested plus one for the final goal vertex tested.

Now notice that $\prob[Z_i] = 1 / (g + 1)$ since there are $(g + 1)!$ ways of ordering the $(g + 1)$ vertices in the set containing $v_i$ and the $g$ goal vertices, and $g!$ of these orderings start with $v_i$.
Importantly, this probability does not change when some other non-goal vertex $v_j \neq v_i$ is picked.
Therefore, the $Z_i$'s are independent and so
\begin{align}
    \expect[X] & = 1 + \expect[\sum_{i = 1}^{\goallevelsize - \numgoals} Z_i] = 1 + \sum_{i = 1}^{\goallevelsize - \numgoals} \expect[Z_i]\\
    & = 1 + (\goallevelsize - \numgoals) / (\numgoals + 1) = (\goallevelsize + 1) / (\numgoals + 1)\label{brfs_ex_proof:ind}
\end{align}
Line \ref{brfs_ex_proof:ind} holds since the $Z_i$ are indicator variables and so $\expect[Z_i]=\prob[Z_i]$, and since we are summing over $(\goallevelsize - \numgoals)$ of them with the same expectation.
Adding this to $N_{\mathcal{O}}$ yields the desired result.
\end{proof}

We now consider the runtime of constant depth RRW.
For this result, we assume that each random walk has an independent probability of $0 \leq s \leq 1$ of reaching a goal vertex.
We call $s$ the \emph{success probability}.

\begin{theorem} \label{thm:rrw_exp}
If $\mathcal{S}$ is a search task with goal level $\goallevel$ such that $s >0 $ and all $\numgoals$ goal vertices are at level $\goallevel$, then 
\begin{align*}
    \expect[\rrwrv_e(\mathcal{S})] = \frac{e \goallevel}{s} - (e - 1)\goallevel + 1
\end{align*}
\end{theorem}
\begin{proof}
Let $Y$ be the random variable for the \emph{number of random walks} it takes to find a goal (\textit{ie.} the goal is first seen on random walk $Y$).
Notice that
\begin{align}
    \expect[\rrwrv_e(\mathcal{S})]  &  = \expect[e \goallevel (Y - 1) + \goallevel + 1] \label{rrw_s:init}\\
            & = e \goallevel \expect[Y] - (e - 1) \goallevel + 1 \label{rrw_s:decompose}
\end{align}
Line \ref{rrw_s:init} holds because the first $Y$ random walks all perform $e \goallevel$ goal tests before restarting and exactly $\goallevel$ tests on the last random walk (when it finds a goal at depth $\goallevel$).
The added $1$ comes from the single goal test of $\initvertex$.

Since every random walk has an independent and identically distributed success probability of $s$, $Y$ follows the well-known geometric distribution.
Thus, $E[Y] = 1/s$.
Substituting this into Line \ref{rrw_s:decompose} gives the desired result.
\end{proof}

\section{Comparing BrFS and RRW on Trees}

We now consider the performance of BrFS and RRW on directed trees with a constant branching factor. 
We begin by using Theorems \ref{thm:brfs_exp} and \ref{thm:rrw_exp} to find more specific expressions for the expected number of goal tests in this case.
\begin{corollary} \label{cor:brfs_trees}
If $\mathcal{S}_\mathcal{T}$ is a search task on a directed tree with constant branching factor $\bfactor \geq 2$ and $\numgoals \geq 1$ goals uniformly distributed at goal level $\goallevel$, then 
\begin{align*}
    \expect[\bfsrv(\mathcal{S}_\mathcal{T})] = (\bfactor^{\goallevel} - 1)/(\bfactor - 1) + (\bfactor^{\goallevel} + 1) / (\numgoals + 1) 
\end{align*}
\end{corollary}
This follows since there are $\bfactor^\ell$ vertices at every level $\ell \geq 0$ of such a tree.
Thus, we can apply Theorem \ref{thm:brfs_exp} using  $\goallevelsize = \bfactor^{\goallevel}$ and $N_{\mathcal{O}} = \bfactor^0 + \bfactor^1 + ... \bfactor^{\goallevel - 1} = (\bfactor^{\goallevel} - 1)/(\bfactor - 1)$, which follows from the formula for the sum of a geometric series.


Next, we consider RRW on such a tree:
\begin{corollary} \label{cor:rrw_trees}
Let $\mathcal{S}_\mathcal{T}$ be a search task on a directed tree with constant branching factor $\bfactor \geq 2$ and $\numgoals \geq 1$ goals all at the goal level $\goallevel$.
If RRW samples edges uniformly at each step, then 
\begin{align*}
    \expect[\rrwrv_e(\mathcal{S}_\mathcal{T})] = \frac{e \goallevel \bfactor^{\goallevel} } {\numgoals} - (e - 1)\goallevel + 1
\end{align*}
\end{corollary}
This follows immediately from Theorem \ref{thm:rrw_exp} since the success probability is $\numgoals / \bfactor^{\goallevel}$ because all vertices at the goal level are equally likely to be visited on any random walk.

\begin{figure*}[t!]
 	\centering
 	\begin{subfigure}[t]{0.3\textwidth}
 		\centering
 		\includegraphics[width=\linewidth]{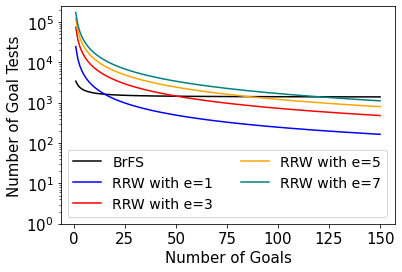} 
        \caption{Expected Goal Tests}\label{fig:goal_behaviour}
 	\end{subfigure}
 	\hfill
 	\begin{subfigure}[t]{0.3\textwidth}
 		\centering
 		\includegraphics[width=\linewidth]{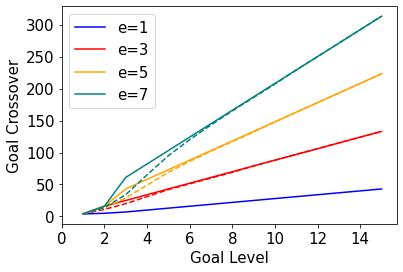} 
        \caption{Goal Level vs. Goal Crossover}\label{fig:cutoff_behaviour}
 	\end{subfigure}
 	\hfill
 	\begin{subfigure}[t]{0.3\textwidth}
 		\centering
 		\includegraphics[width=\linewidth]{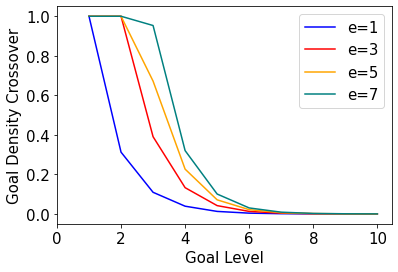} 
        \caption{Goal Level vs. Goal Density Crossover}\label{fig:density_behaviour}
 	\end{subfigure}
 	\caption{Comparing BrFS and RRW with different depth errors ($e$) on a directed tree with a branching factor of 4. Figure \ref{fig:goal_behaviour} shows the expected goal tests when the goal depth is 6 for different numbers of goals. Figure \ref{fig:cutoff_behaviour} shows how the crossover point changes at different goal levels. The crossover point as computed empirically is shown as a dashed line, while the bound from Theorem \ref{thm:brfs_vs_rrw} is a solid line. Figure \ref{fig:density_behaviour} shows how the goal density crossover changes with the goal level.}
\end{figure*}

Figure \ref{fig:goal_behaviour} uses these corollaries to show the expected number of goal tests for BrFS and RRW on a tree with $\bfactor=4$, $\goallevel=6$, and different numbers of goals.
BrFS performs significantly fewer goal tests for a small number of goals, but RRW can overtake BrFS quite quickly depending on the error value (note that there are 4096 vertices at the goal level).
We refer to the number of goals at which RRW matches or surpasses BrFS as the \emph{crossover point}.
The next theorem shows that the crossover point is linear in each of the branching factor, goal depth, and depth error.
The theorem only covers the case where the $\goallevel \geq 2$ and $e \goallevel > 2$. 
This requires that either the goal level $\goallevel > 2$ or the random walk depth error is $e > 1$.
We discuss the remaining cases after the theorem.
\begin{theorem} \label{thm:brfs_vs_rrw}
    Let $\mathcal{S}_\mathcal{T}$ be a search task on a directed tree with constant branching factor $\bfactor \geq 2$, $\goallevel \geq 2$, and $\numgoals \geq 1$ goals all uniformly distributed at the goal level.
    If RRW samples edges uniformly at each step and $e \goallevel > 2$, then $\expect[\bfsrv(\mathcal{S}_\mathcal{T})] \geq \expect[\rrwrv(\mathcal{S}_\mathcal{T})]$ if either the following holds:
    \begin{enumerate}
        \item $\numgoals = \bfactor^{\goallevel}$
        \item $\numgoals < \bfactor^{\goallevel}$ and $\numgoals \geq (e \goallevel - 1) (\bfactor - 1) + 1$ \label{bfs_vs_rrw_line:goal_bound}
    \end{enumerate}
\end{theorem}
\begin{proof}
Let $x$ denote $\expect[\bfsrv(\mathcal{S}_\mathcal{T})] - \expect[\rrwrv(\mathcal{S}_\mathcal{T})]$.
By Corollaries \ref{cor:brfs_trees} and \ref{cor:rrw_trees}
\begin{align}
    x = \frac{\bfactor^{\goallevel} - 1}{\bfactor - 1} + \frac{\bfactor^{\goallevel} + 1}{\numgoals + 1} - \frac{e \goallevel \bfactor^{\goallevel}}{\numgoals} + (e - 1)\goallevel - 1 \label{bfs_vs_rrw_line:x_def}
\end{align}
We will now show $x\geq 0$ in case 1, when $\numgoals = \bfactor^{\goallevel}$. Notice that this simplifies line \ref{bfs_vs_rrw_line:x_def} to $x = (\bfactor^{\goallevel} - 1)/(\bfactor - 1) - \goallevel$.
Since $\bfactor \geq 2$ and $\goallevel \geq 2$, $x$ is clearly positive here.

Let us now consider Case \ref{bfs_vs_rrw_line:goal_bound}.
For readability, we let $x = y + z$, where $z = (e - 1)\goallevel - 1$, and $y$ is the remaining terms in line \ref{bfs_vs_rrw_line:x_def}.
Since $x \geq 0$ if and only if $C x \geq 0$ for some positive constant $C$, we can show the required statement by showing $x' \geq 0$ where $x' = cx$ and $c = (\bfactor - 1)(\numgoals + 1) \numgoals$.
This is done to avoid the denominators.
Similarly, we define $y' = c y$ and $z' = c z$ and derive the following:
\begin{align}
    y' & = (\bfactor^{\goallevel} - 1)(\numgoals + 1) \numgoals + (\bfactor^{\goallevel} + 1)(\bfactor - 1) \numgoals \\ \nonumber
    &~~~~~~ - e \goallevel \bfactor^{\goallevel} (\bfactor - 1)(\numgoals + 1)\\
    & = \bfactor^{\goallevel} [(\numgoals + 1) \numgoals + (\bfactor - 1) \numgoals - e \goallevel (\bfactor - 1)(\numgoals + 1)] + v \label{bfs_vs_rrw_line:factor}
\end{align}
where $v = (\bfactor - 1) \numgoals - (\numgoals + 1) \numgoals$. Here, we have just separated and factored out the terms with $\bfactor^{\goallevel}$ in them.
Now notice the lower bound in Case \ref{bfs_vs_rrw_line:goal_bound} is equivalent to $ed(b - 1) \leq \numgoals + \bfactor + 2$.
By substituting this into Line \ref{bfs_vs_rrw_line:factor}, we get 

\begin{align}
     y' & \geq \bfactor^{\goallevel} [(\numgoals + 1) \numgoals + (\bfactor - 1) \numgoals - (\numgoals + \bfactor + 2) (\numgoals + 1)] + v \label{bfs_vs_rrw_line:sub_and_algebra}\\
     & \geq \bfactor^{\goallevel}(\numgoals - \bfactor - 2) + v \geq \bfactor^{\goallevel}(\bfactor - 1) + v
\end{align}
Line \ref{bfs_vs_rrw_line:sub_and_algebra} merely involves expanding and cancelling out terms.
The last line holds by the assumption that $e \goallevel  - 1 \geq 3$ which, alongside the lower bound in Case \ref{bfs_vs_rrw_line:goal_bound} implies that $g \geq 2\bfactor - 1$.

\begin{align}
    y' + z' & \geq \bfactor^{\goallevel}(\bfactor - 1) + (\bfactor - 1) \numgoals - (\numgoals + 1)\numgoals\\ \nonumber
    & ~~~~~~ + (\bfactor - 1)(\numgoals + 1) \numgoals [(e - 1)\goallevel - 1]\\
    & \geq \bfactor^{\goallevel}(\bfactor - 1) + \numgoals[(\bfactor - 1)(\numgoals + 1)(e - 1)\goallevel - 1] \label{bfs_vs_rrw_line:v_algebra}\\ 
    & \geq \bfactor^{\goallevel}(\bfactor - 1) - \numgoals  \label{bfs_vs_rrw_line:g_lower} \\
    & \geq \bfactor^{\goallevel}(\bfactor - 1) - \bfactor^{\goallevel} \label{bfs_vs_rrw_line:y_plus_z_bound}
\end{align}
Line \ref{bfs_vs_rrw_line:v_algebra} simply involves expanding and cancelling out terms.
Line \ref{bfs_vs_rrw_line:g_lower} holds because $(\bfactor - 1)(\numgoals + 1)(e - 1) \geq 0$ since $\bfactor \geq 2$, $\numgoals \geq 1$, and $e \geq 1$.
The last line holds since $\numgoals < \bfactor^{\goallevel}$.

Since $\bfactor \geq 2$, line \ref{bfs_vs_rrw_line:y_plus_z_bound} implies that $x' = y' + z' \geq 0$, which in turn implies $x \geq 0$, which completes the proof.
\end{proof}

Let us now consider cases not handled by Theorem \ref{thm:brfs_vs_rrw}. 
When $\goallevel = 1$, it can be shown that $\expect[\bfsrv(\mathcal{S}_\mathcal{T})] < \expect[\rrwrv(\mathcal{S}_\mathcal{T})]$ for any $\numgoals \leq \bfactor^{\goallevel} = \bfactor$, and the two expectations are exactly equal when $\numgoals = \bfactor$.
If $\goallevel = 2$ and $e = 1$, the situation is almost identical to that in Theorem \ref{thm:brfs_vs_rrw} except the lower bound on $g$ in Case \ref{bfs_vs_rrw_line:goal_bound} is exactly one higher than in Theorem \ref{thm:brfs_vs_rrw}.
Namely, the bound is
$\numgoals \geq (e \goallevel - 1) (\bfactor - 1) + 2$.

Figure \ref{fig:cutoff_behaviour} shows how the goal crossover changes with goal level and depth error. 
The bound from Theorem \ref{thm:brfs_vs_rrw} is shown as a solid line and the actual crossover points  --- computed based on Corollaries \ref{cor:brfs_trees} and \ref{cor:rrw_trees} --- are shown as dashed lines.
The figure clearly depicts that the crossover increases linearly with the goal level and depth error as suggested by Theorem \ref{thm:brfs_vs_rrw}.
The figure also shows that the bound is quite accurate, as it only slightly over-estimates the true crossover for small goal levels.

Finally, we note that while the crossover increases linearly with the goal level, the density of goals needed at the goal level in order for RRW to outperform BrFS actually decreases with $\goallevel$.
For example, Figure \ref{fig:density_behaviour} shows the \emph{density crossover}, which we define as the crossover point (determined using the lower bound from Theorem \ref{thm:brfs_vs_rrw}) divided by the number of goals at the goal level (namely $\bfactor^{\goallevel}$).
The figure clearly shows that the required percentage of goals for RRW to be faster actually decreases with the goal level.

\section{Discussion and Future Work} \label{sec:discussion}

We have shown that BrFS performs fewer goal tests when the number of goals is very small, but RRW is faster in expectation as the number of goals grows even slightly.
However, RRW may outperform this result in practice for several reasons.
First, RRW can benefit from goals at levels between $\goallevel$ and $e \goallevel$ as these will increase the success probability and thereby decrease expected runtime.
BrFS cannot benefit from such goals in any way.

Second, while we have focused on the number of goal tests performed, the amount of time needed per goal test differs between the algorithms.
This is because BrFS has the additional overhead of maintaining the open and closed list (\textit{ie.} duplicate checking, adding and removing from the open list, etc.).
BrFS will also generate all children of a given vertex, not just a single one as RRW does.
While this overhead may just be constant time, that may often be enough to further decrease the crossover point when expected runtime and not number of goal tests are considered.
The open and closed lists can also have enormous memory requirements making BrFS impractical in certain situations.

In contrast, the duplicate detection performed by open and closed lists means BrFS is much better equipped to handle graphs with \emph{cycles} or \emph{transpositions}.
In practice, this effectively means RRW is working on a larger search tree.
For example, even if the search task is an undirected tree with constant branching factor $\bfactor$, the duplicate detection of BrFS means it will be able to prune the tree to one with a branching factor $\bfactor - 1$.
While this particular case is easy to handle with RRW, this algorithm will struggle relative to BrFS in a similar way to iterative deepening methods \cite{korf:ida_star} if the transpositions are more complex than that.
Along with our results in Theorem \ref{thm:brfs_vs_rrw}, we therefore expect RRW to have the advantage in large combinatorial state-spaces with a low number of transpositions and an increasing amount of goals, and BrFS to better handle cases with very few goals or many transpositions.
It may therefore be important to develop hybrid BrFS and RRW methods --- analogous to A* + IDA* \cite{bu:a_star_plus_ida_star} --- that get some of the advantages of both.
We leave such an investigation as future work.

We now review some related work and identify several avenues of future work.
Local best-first searches have been empirically shown to be useful for escaping local minima and plateaus in planning search systems \cite{xie:improved_local}.
\citeauthor{nakhost:rw_analysis} (\citeyear{nakhost:rw_analysis}) formally analyze the expected runtime of a single random walk and restarting random walks on classes of graphs characterized by the probability of getting closer or farther from a goal on every step.
Their model for restarting random walks assumed a constant restart probability at every step instead of a constant restart depth, and they also do not compare to BrFS.
Extending our results to their graph classes is left as future work.
\citeauthor{everitt:bfs_vs_dfs} (\citeyear{everitt:bfs_vs_dfs}) performed a similar analysis to ours when comparing BrFS and depth-first search (DFS) on bounded depth-trees.
They characterized the expected runtime of BrFS, but their formulation considers the case where every vertex at the goal level is equally likely to be a goal (see Proposition 3).
Our use of a fixed number of goals $\numgoals$ considerably simplifies the resulting expression for the expected runtime, and we also compare against RRW instead of DFS, which is not generally suitable for escaping local minima/plateaus.

Finally, non-constant restart policies have been introduced that do not rely on an initial guess of the goal depth.
These methods dynamically adjust the restart depth over time while still maintaining strong theoretical guarantees.
For example, the general policy of \citeauthor{luby:lv_restarts} (\citeyear{luby:lv_restarts}) has an expected runtime of $O(T \log T)$ on any search task $\mathcal{S}$, where $T$ is the runtime of the best possible restart policy for $\mathcal{T}$.
As on directed trees, the best possible policy is exactly the constant-depth one with $e=1$, we hope to extend our analysis to compare BrFS to this general policy.
We leave this as future work.

\bibliography{aaai24}

\end{document}

%% file: macros.tex
\newcommand{\expect}{\mathbb{E}}
\newcommand{\prob}{\mathbb{P}}

\newcommand{\tuple}[1]{\langle{#1}\rangle}

\newcommand{\goalset}{V_{\mathcal{G}}}
\newcommand{\initvertex}{v_{\mathcal{I}}}
\newcommand{\goallevel}{d^*}
\newcommand{\goallevelsize}{N_{\goallevel}}
\newcommand{\numgoals}{g}
\newcommand{\bfactor}{b}
\newcommand{\bfsrv}{B}
\newcommand{\rrwrv}{R}